\documentclass{article}

\usepackage{times}
\usepackage{graphicx} 
\usepackage{subfigure} 
\usepackage{sidecap}

\usepackage[latin1]{inputenc}
\usepackage{amsmath}
\usepackage{amsfonts}
\usepackage{amssymb}

\usepackage{bbm}

\usepackage{natbib}

\usepackage{algorithm}
\usepackage{algorithmic}

\usepackage{url}

\usepackage[accepted]{arxiv}

\newtheorem{theorem}{Theorem}

\newtheorem{proposition}[theorem]{\bf{Proposition}}

\newcommand{\vect}[1]{\mathbf{#1}}
\newcommand{\set}[1]{\mathbb{#1}}
\newcommand{\sett}[1]{\mathcal{#1}}
\newcommand{\matr}[1]{\mathbf{#1}}
\newcommand{\vectsymb}[1]{\boldsymbol{#1}}

\newcommand{\DP}[2]{{#1}^{\top}{#2}}

\DeclareMathOperator*{\argmin}{argmin}

\icmltitlerunning{From Softmax to Sparsemax: A Sparse Model of Attention and Multi-Label Classification}

\begin{document} 

\twocolumn[
\icmltitle{From Softmax to Sparsemax:\\A Sparse Model of Attention and Multi-Label Classification}


\icmlauthor{Andr\'{e}~F.~T.~Martins$^{\dagger\sharp}$}{andre.martins@unbabel.com}
\icmlauthor{Ram\'{o}n F. Astudillo$^{\dagger\diamond}$}{ramon@unbabel.com}
\icmladdress{$^{\dagger}$Unbabel Lda, Rua Visconde de Santar\'em, 67-B, 1000-286 Lisboa, Portugal\\ $^{\sharp}$Instituto de Telecomunica\c{c}\~oes (IT), Instituto Superior T\'ecnico, Av. Rovisco Pais, 1, 1049-001 Lisboa, Portugal\\$^{\diamond}$Instituto de Engenharia de Sistemas e Computadores (INESC-ID), Rua Alves Redol, 9, 1000-029 Lisboa, Portugal}

\icmlkeywords{boring formatting information, machine learning, ICML}

\vskip 0.3in
]

\begin{abstract} 
We propose sparsemax, a new activation function similar to the traditional softmax, but able to output sparse probabilities. 
After deriving its properties, we show how its Jacobian can be efficiently computed, 
enabling its use in a network trained with backpropagation. 
Then, we propose a new smooth and convex loss function which is the sparsemax analogue of the logistic loss. 
We reveal an unexpected connection between this new loss and the 
Huber classification loss. 
We obtain promising empirical results in multi-label classification problems and in attention-based neural networks for natural language inference. For the latter, we achieve a similar performance as the traditional softmax, but with a 
selective, more compact, attention focus.
\end{abstract} 

\section{Introduction}

The softmax transformation is a key component of several statistical learning models, 
encompassing
multinomial logistic regression \citep{McCullagh1989}, 
action selection in reinforcement learning \citep{Sutton1998}, 
and neural networks for multi-class classification \citep{Bridle1990,Goodfellow2016}.  
Recently, 
it 
has also been used 
to design attention mechanisms in neural networks, 
with 
important achievements in 
machine translation \citep{Bahdanau2015}, image caption generation \citep{Xu2015}, 
speech recognition \citep{Chorowski2015}, 
memory networks \citep{Sukhbaatar2015},
and various tasks in natural language understanding \citep{Hermann2015,Rocktaschel2015,Rush2015EMNLP} 
and computation learning \citep{Graves2014,Grefenstette2015}. 

There are a number of reasons why the softmax transformation is so appealing.  
It is simple to evaluate and differentiate, 
and it can be turned into the 
(convex) negative log-likelihood loss function by taking the logarithm of its output. 
Alternatives proposed in the literature, such as the 
Bradley-Terry model 
\citep{Bradley1952,
Zadrozny2001,Menke2008}, 
the multinomial probit \citep{Albert1993}, 
the spherical softmax \citep{Ollivier2013,Vincent2014,Brebisson2015}, 
or softmax approximations \citep{Bouchard2007}, 
while theoretically or computationally advantageous for certain scenarios, 
lack some of the convenient properties of softmax. 

In this paper, we 
propose 
the {\textbf{sparsemax transformation}}. 
Sparsemax 
has the distinctive feature that it can return sparse posterior distributions, 
that is, it may assign exactly zero probability to some of its output variables. 
This property makes it appealing to be used as a filter for large output spaces, to predict multiple labels, or as a 
component to identify which of a group of variables are potentially relevant for a decision,   
making the model more interpretable.  
Crucially, this is done while preserving most of the attractive properties of softmax: we show that sparsemax is also simple to evaluate, it is even cheaper to differentiate, 
and that it can be turned into a convex loss function. 

To sum up, our contributions are as follows:
\begin{itemize}
\item We formalize the new sparsemax transformation, derive its properties, and show how it can be efficiently computed (\S\ref{sec:sparsemax_definition}--\ref{sec:sparsemax_properties}). We show that in the binary case sparsemax reduces to a hard sigmoid (\S\ref{sec:sparsemax_2D_3D}).
\item We derive the Jacobian of sparsemax, comparing it to the softmax case, 
and show that it can lead to faster gradient backpropagation (\S\ref{sec:sparsemax_jacobian}).
\item We propose the {\textbf{sparsemax loss}}, a new loss function that is the sparsemax analogue of logistic regression (\S\ref{sec:sparsemax_loss}). We show that it is convex, everywhere differentiable, 
and can be regarded as a
multi-class generalization of the Huber classification loss, an important tool in robust statistics \citep{Huber1964,Zhang2004}.
\item We apply the sparsemax loss to train {multi-label} linear classifiers (which predict a \emph{set} of labels instead of a single label) on benchmark datasets (\S\ref{sec:experiments_simulations}--\ref{sec:experiments_multilabel}).
\item Finally, we devise a neural selective attention mechanism using the sparsemax transformation, 
evaluating its performance on a natural language inference problem, 
with encouraging results (\S\ref{sec:experiments_attention}).
\end{itemize}

%

\section{The Sparsemax Transformation}\label{sec:sparsemax}

\subsection{Definition}\label{sec:sparsemax_definition}

Let $\Delta^{K-1} := \{\vectsymb{p}\in \mathbb{R}^K \,\,|\,\, \DP{\vect{1}}{\vectsymb{p}} = 1, \,\, \vectsymb{p}\ge \vect{0}\}$ be 
the $(K-1)$-dimensional simplex. 
We are interested in functions that map vectors in $\mathbb{R}^K$ to probability distributions in $\Delta^{K-1}$. 
Such functions are useful for converting a vector of real weights (\emph{e.g.}, label scores) to a probability distribution (\emph{e.g.} posterior probabilities of labels). 
The classical example is the \textbf{softmax function}, 
defined componentwise as:
\begin{equation}\label{eq:softmax}
\mathrm{softmax}_i(\vectsymb{z}) = \frac{\exp(z_i)}{\sum_{j} \exp(z_j)}.
\end{equation} 
A limitation of the softmax transformation is that 
the resulting probability distribution always has full support, \emph{i.e.}, 
$\mathrm{softmax}_i(\vectsymb{z}) \ne 0$ for every $\vectsymb{z}$ and $i$. 
This is a disadvantage in applications where 
a sparse probability distribution is desired, in 
which case it is common to define a threshold 
below which small probability values are truncated to zero. 

In this paper, we propose as an alternative the following  transformation, which we call \textbf{sparsemax}:
\begin{equation}\label{eq:sparsemax}
\mathrm{sparsemax}(\vectsymb{z}) := \argmin_{\vectsymb{p} \in \Delta^{K-1}} \|\vectsymb{p} - \vectsymb{z}\|^2.
\end{equation}
In words, sparsemax returns the Euclidean projection of the input vector $\vectsymb{z}$ onto the probability simplex. This projection is likely to hit the boundary of the simplex, in which case $\mathrm{sparsemax}(\vectsymb{z})$ becomes sparse. We will see that sparsemax retains most of the important properties of softmax, having in addition the ability of producing 
sparse distributions. 

\subsection{Closed-Form Solution}\label{sec:sparsemax_closedform}

Projecting onto the simplex is a well studied problem, 
for which linear-time algorithms are available     
\citep{Michelot1986,Pardalos1990,Duchi2008}. 
We start by recalling the well-known result that 
such projections correspond to a {soft-thresholding} operation. 
Below, we use the notation $[K] := \{1,\ldots,K\}$ and $[t]_+ := \max\{0, t\}$. 
\begin{proposition}\label{prop:sparsemax_solution} 
The solution of Eq.~\ref{eq:sparsemax} is of the form: 
\begin{equation}\label{eq:sparsemax_closedform}
\mathrm{sparsemax}_i(\vectsymb{z}) = [z_i - \tau(\vectsymb{z})]_+,
\end{equation}
where $\tau : \mathbb{R}^K \rightarrow \mathbb{R}$ is the (unique) function that satisfies 
$\sum_j [z_j - \tau(\vectsymb{z})]_+ = 1$ for every $\vectsymb{z}$. 
Furthermore, $\tau$ 
can be expressed as follows. 
Let $z_{(1)} \ge z_{(2)} \ge \ldots \ge z_{(K)}$ be the sorted coordinates of $\vectsymb{z}$, and define 
$k(\vectsymb{z}) := \max\left\{k \in [K]\,\,|\,\, 1+kz_{(k)} > \sum_{j\le k} z_{(j)}\right\}$.
Then,
\begin{eqnarray}\label{eq:threshold_closedform}
\tau(\vectsymb{z}) = \frac{\left(\sum_{j\le k(\vectsymb{z})} z_{(j)}\right) - 1}{k(\vectsymb{z})}
= 
\frac{\left(\sum_{j \in S(\vectsymb{z})} z_j\right) - 1}{|S(\vectsymb{z})|},
\end{eqnarray}
where $S(\vectsymb{z}) := \{j \in [K]\,\,|\,\, \mathrm{sparsemax}_j(\vectsymb{z}) > 0\}$ is the support of 
$\mathrm{sparsemax}(\vectsymb{z})$.
\end{proposition}
\begin{proof}
See App.~\ref{sec:proof_threshold_closedform} in the supplemental material. 
\end{proof}

In essence, Prop.~\ref{prop:sparsemax_solution} states that all we need for evaluating the sparsemax transformation is to compute the threshold $\tau(\vectsymb{z})$; 
all coordinates above this threshold (the ones in the set $S(\vectsymb{z})$) will be shifted by this amount, 
and the others will be truncated to zero.  
We call $\tau$ in Eq.~\ref{eq:threshold_closedform} 
the \textbf{threshold function}. This piecewise linear function will play an important role in the sequel. 
Alg.~\ref{alg:projontosimplex} illustrates a na\"ive $O(K\log K)$ algorithm that uses 
Prop.~\ref{prop:sparsemax_solution} for evaluating the sparsemax.%
\footnote{More elaborate $O(K)$ algorithms exist based on linear-time selection \citep{Blum1973,Pardalos1990}.} %
\begin{algorithm}[t]
   \caption{Sparsemax Evaluation}
\begin{algorithmic}\label{alg:projontosimplex}
   \STATE {\bfseries Input:} $\vectsymb{z}$
   \STATE Sort $\vectsymb{z}$ as $z_{(1)} \ge \ldots \ge z_{(K)}$
   \STATE Find $k(\vectsymb{z}) := \max\left\{k \in [K]\,\,|\,\, 1+kz_{(k)} > \sum_{j\le k} z_{(j)}\right\}$
   \STATE Define $\tau(\vectsymb{z}) = \frac{\left(\sum_{j\le k(\vectsymb{z})} z_{(j)}\right) - 1}{k(\vectsymb{z})}$
   \STATE {\bfseries Output:} $\vectsymb{p}$ s.t. $p_i = [z_{i} - \tau(\vectsymb{z})]_+$.
\end{algorithmic}
\end{algorithm}

\begin{figure*}[t]
\centering
\includegraphics[width=0.27\textwidth]{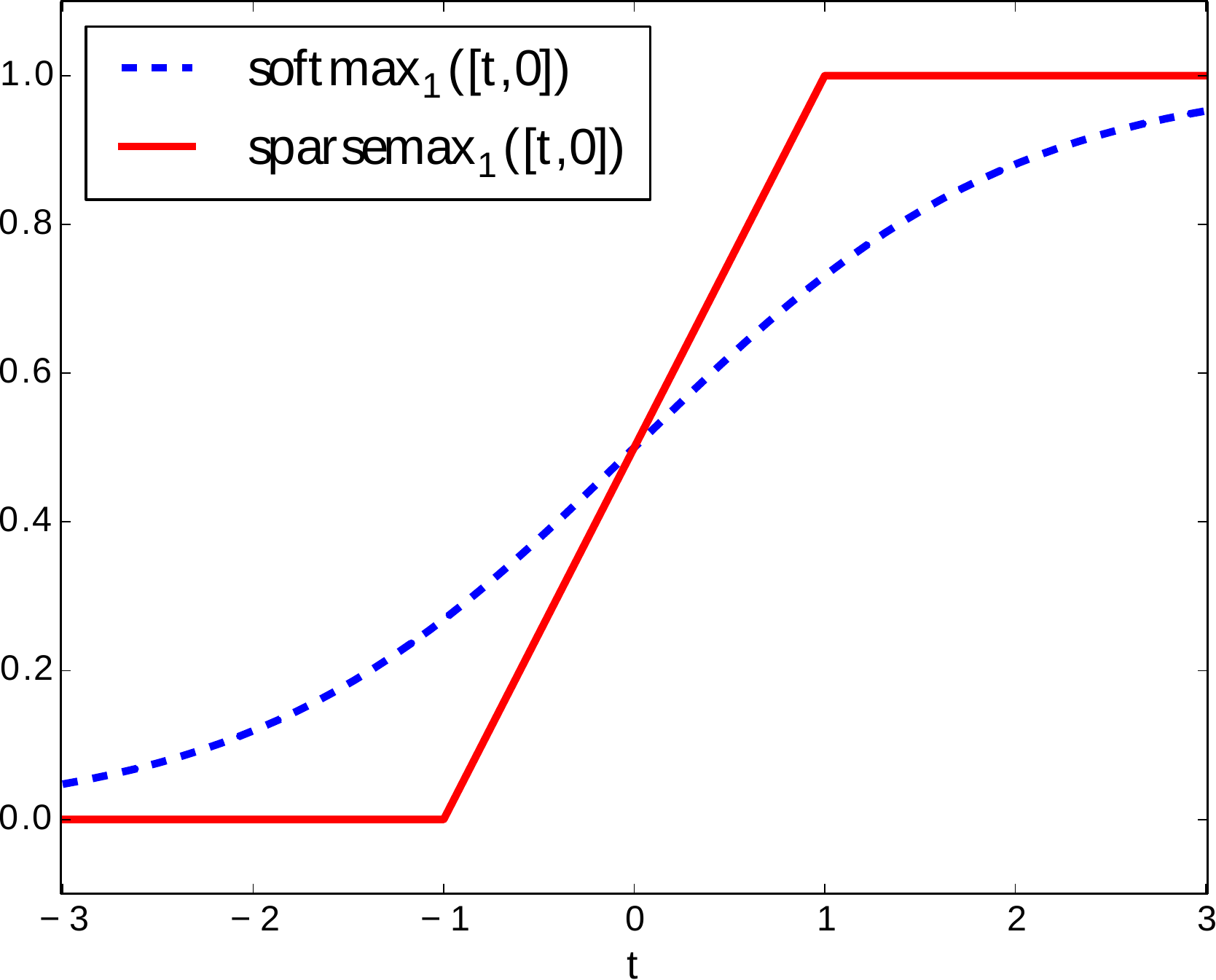}
\includegraphics[width=0.65\textwidth]{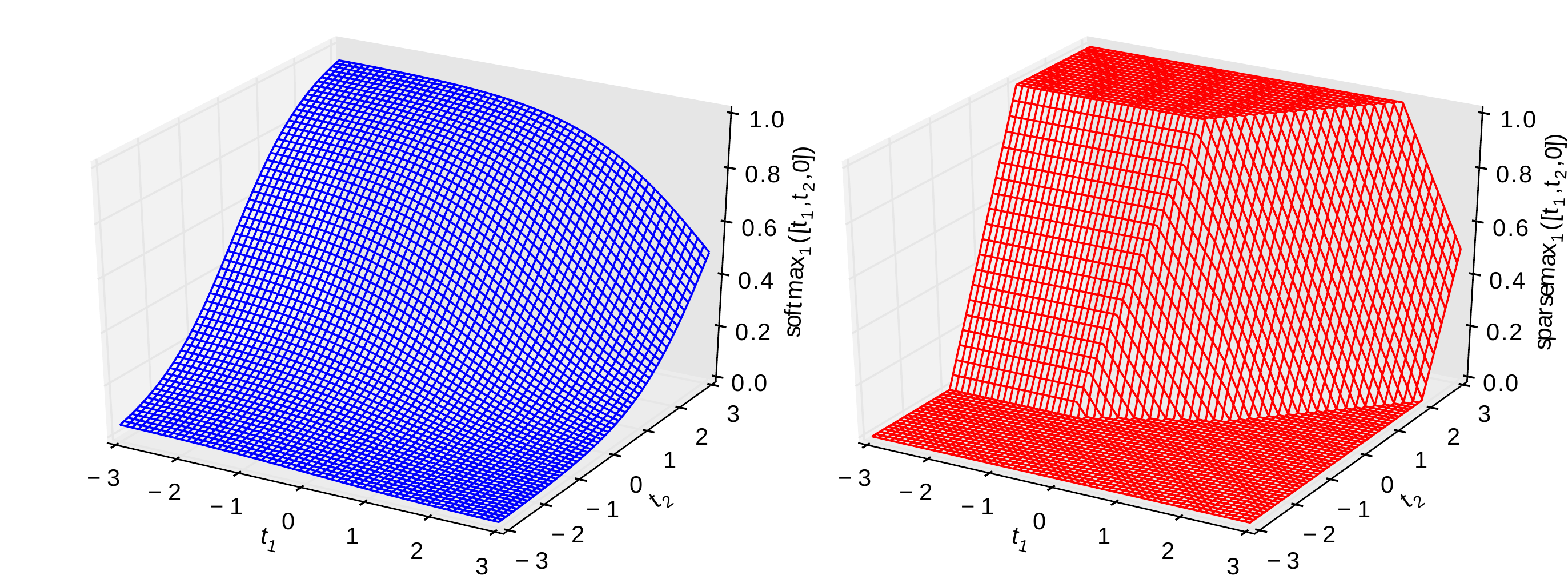}
\caption{Comparison of softmax and sparsemax in 2D (left) and 3D (two righmost plots).\label{fig:softmax_sparsemax_2d_3d}}
\end{figure*}

\subsection{Basic Properties}\label{sec:sparsemax_properties}

We now highlight some properties that are common to softmax and sparsemax. 
Let $z_{(1)} := \max_k z_k$, and denote by $A(\vectsymb{z}) := \{k \in [K]\,\,|\,\, z_k = z_{(1)}\}$ the set of maximal components of $\vectsymb{z}$. We define 
the indicator vector $\mathbbm{1}_{A(\vectsymb{z})}$, whose $k$th component is $1$ if $k\in A(\vectsymb{z})$, and $0$ otherwise. We further denote by $\gamma(\vectsymb{z}) := z_{(1)} - \max_{k\notin A(\vectsymb{z})} z_k$ the gap between the maximal components of $\vectsymb{z}$ and the second largest. 
We let $\vect{0}$ and $\vect{1}$ be vectors of zeros and ones, respectively.


\begin{proposition}\label{prop:sparsemax_softmax_properties} 
The following properties hold for $\rho \in \{\mathrm{softmax}, \mathrm{sparsemax}\}$.
\begin{enumerate}
\item $\rho(\vect{0}) = \vect{1}/K$ and $\lim_{\epsilon\rightarrow 0^+}\rho(\epsilon^{-1}\vectsymb{z}) = \mathbbm{1}_{A(\vectsymb{z})} / |A(\vectsymb{z})|$ (uniform distribution, and distribution peaked on the maximal components of $\vectsymb{z}$, respectively). 
For sparsemax, the last equality  holds for any $\epsilon \le \gamma(\vectsymb{z})\cdot |A(\vectsymb{z})|$.  
\item $\rho(\vectsymb{z}) = \rho(\vectsymb{z} + c\vect{1})$, for any $c \in \set{R}$ (\emph{i.e.}, $\rho$  is invariant to adding a constant to each coordinate).
\item $\rho(\matr{P}\vectsymb{z}) = \matr{P}\rho(\vectsymb{z})$ for any permutation matrix $\matr{P}$ (\emph{i.e.}, $\rho$ commutes with permutations).
\item If $z_i\le z_j$, then $0 \le \rho_j(\vectsymb{z})-\rho_i(\vectsymb{z}) \le \eta(z_j - z_i)$, 
where  
$\eta=\frac{1}{2}$ for 
softmax, and $\eta=1$ for sparsemax. 
\end{enumerate} 
\end{proposition}
\begin{proof}
See App.~\ref{sec:proof_prop_sparsemax_softmax_properties} in the supplemental material. 
\end{proof}

Interpreting $\epsilon$ as a ``temperature parameter,'' the first part of Prop.~\ref{prop:sparsemax_softmax_properties} 
shows that the sparsemax has the same ``zero-temperature limit'' behaviour as 
the softmax, but without the need of making the temperature arbitrarily small. 
%

Prop.~\ref{prop:sparsemax_softmax_properties} is reassuring, since it shows that the sparsemax transformation, 
despite being defined very differently from the softmax, 
has a similar behaviour and preserves the same invariances. 
Note that some of these properties are not satisfied by other proposed replacements of the softmax: 
for example, the spherical softmax \citep{Ollivier2013}, defined as 
$\rho_i(\vectsymb{z}) := z_i^2 / \sum_j z_j^2$, 
does not satisfy properties 2 and 4. 

\subsection{Two and Three-Dimensional Cases}\label{sec:sparsemax_2D_3D}

For the two-class case, 
it is well known that the softmax activation becomes the logistic (sigmoid) function. More precisely, 
if $\vectsymb{z} = (t,0)$, %
then 
$\mathrm{softmax}_1(\vectsymb{z}) = \sigma(t) := {(1+\exp(-t))^{-1}}$. 
We next show that the analogous in sparsemax is the ``hard'' version of the sigmoid. 
In fact, using Prop.~\ref{prop:sparsemax_solution}, Eq.~\ref{eq:threshold_closedform}, we have that, for $\vectsymb{z} = (t,0)$, 
\begin{equation}
\tau(\vectsymb{z}) = \left\{
\begin{array}{ll}
t-1, & \text{if $t>1$} \\
(t-1)/2, & \text{if $-1\le t \le 1$} \\
-1, & \text{if $t<-1$},
\end{array}
\right.
\end{equation}
and therefore
\begin{equation}
\mathrm{sparsemax}_1(\vectsymb{z}) = 
\left\{
\begin{array}{ll}
1, & \text{if $t>1$} \\
(t+1)/2, & \text{if $-1\le t \le 1$} \\
0, & \text{if $t<-1$}.
\end{array}
\right.
\end{equation}
Fig.~\ref{fig:softmax_sparsemax_2d_3d} provides an illustration for the two and three-dimensional cases. 
For the latter, we parameterize  $\vectsymb{z} = (t_1,t_2,0)$ and plot 
$\mathrm{softmax}_1(\vectsymb{z})$ and $\mathrm{sparsemax}_1(\vectsymb{z})$ as a function of $t_1$ and $t_2$.  
We can see that sparsemax is piecewise linear, but asymptotically similar to the softmax.

\subsection{Jacobian of Sparsemax}\label{sec:sparsemax_jacobian}

The Jacobian matrix of a transformation $\rho$, $J_{\rho}(\vectsymb{z}) := [{\partial \rho_i(\vectsymb{z})}/{\partial z_j}]_{i,j}$, 
is of key importance to 
train models with gradient-based optimization. 
We next derive the Jacobian of the sparsemax activation, 
but 
before doing so, let us recall how the Jacobian of the softmax looks like. 
We have
\begin{eqnarray}\label{eq:softmax_gradient}
\!\!\!\!\frac{\partial \mathrm{softmax}_i(\vectsymb{z})}{\partial z_j} 
&\!\!\!\!=\!\!\!\!& \frac{\delta_{ij}e^{z_i}\sum_k e^{z_k} - e^{z_i}e^{z_j}}{\left(\sum_k e^{z_k}\right)^2} \nonumber\\
&\!\!\!\!=\!\!\!\!&\mathrm{softmax}_i(\vectsymb{z}) (\delta_{ij}-\mathrm{softmax}_j(\vectsymb{z})),
\end{eqnarray}
where $\delta_{ij}$ is the Kronecker delta, which evaluates to $1$ if $i=j$ and $0$ otherwise. 
%
Letting $\vectsymb{p} = \mathrm{softmax}(\vectsymb{z})$, 
the full Jacobian can be written in matrix notation as 
\begin{equation}\label{eq:jacobian_softmax}
\matr{J}_{\mathrm{softmax}}(\vectsymb{z}) = \mathbf{Diag}(\vectsymb{p}) -\vectsymb{p} \vectsymb{p}^{\top},
\end{equation} 
where $\mathbf{Diag}(\vectsymb{p})$ is a matrix with $\vectsymb{p}$ in the main diagonal.

Let us now turn to the sparsemax case. 
The first thing to note is that sparsemax is differentiable everywhere except at splitting points $\vectsymb{z}$ where the support set $S(\vectsymb{z})$ changes, \emph{i.e.}, 
where $S(\vectsymb{z}) \ne S(\vectsymb{z}+\epsilon\vectsymb{d})$ for some $\vectsymb{d}$ and infinitesimal $\epsilon$.%
\footnote{For those points, we can take an arbitrary matrix in the set of 
generalized Clarke's Jacobians \citep{Clarke1983}, 
the convex hull of all points of the form $\lim_{t\rightarrow\infty} \matr{J}_{\mathrm{sparsemax}}(\vectsymb{z}_t)$, where $\vectsymb{z}_t\rightarrow \vectsymb{z}$.} 
From Eq.~\ref{eq:sparsemax_closedform}, we have that:
\begin{equation}
\!\!\!\!\frac{\partial \mathrm{sparsemax}_i(\vectsymb{z})}{\partial z_j} = \left\{
\begin{array}{ll}
\delta_{ij} - \frac{\partial \tau(\vectsymb{z})}{\partial z_j},  & \text{if $z_i > \tau(\vectsymb{z})$},\\
0, & \text{if $z_i \le \tau(\vectsymb{z})$.}
\end{array}
\right.
\end{equation}
It remains to compute the gradient of the threshold function $\tau$. 
From Eq.~\ref{eq:threshold_closedform}, we have:
\begin{equation}\label{eq:gradient_tau}
\frac{\partial \tau(\vectsymb{z})}{\partial z_j} = \left\{
\begin{array}{ll}
\frac{1}{|S(\vectsymb{z})|} & \text{if $j \in S(\vectsymb{z})$},\\
0, & \text{if $j \notin S(\vectsymb{z})$}.
\end{array}
\right.
\end{equation}
Note that $j \in S(\vectsymb{z}) \,\, \Leftrightarrow z_j > \tau(\vectsymb{z})$. 
Therefore we obtain:
\begin{equation}\label{eq:sparsemax_gradient}
\!\!\!\!\frac{\partial \mathrm{sparsemax}_i(\vectsymb{z})}{\partial z_j} = \left\{
\begin{array}{ll}
\delta_{ij} - \frac{1}{|S(\vectsymb{z})|}, \!\!\!\! & \text{if $i,j \in S(\vectsymb{z})$},\\
0, & \text{otherwise.}
\end{array}
\right.
\end{equation}
Let $\vectsymb{s}$ be an indicator vector whose 
$i$th entry is $1$ if $i \in S(\vectsymb{z})$, and $0$ otherwise.  
We can write the Jacobian matrix as 
\begin{equation}\label{eq:jacobian_sparsemax}
\matr{J}_{\mathrm{sparsemax}}(\vectsymb{z}) = \mathbf{Diag}(\vectsymb{s}) -\vectsymb{s}\vectsymb{s}^{\top}/{|S(\vectsymb{z})|}.
\end{equation} 
It is instructive to compare Eqs.~\ref{eq:jacobian_softmax} and \ref{eq:jacobian_sparsemax}. 
We may regard {the Jacobian of sparsemax as the Laplacian of a graph whose elements of 
$S(\vectsymb{z})$ are fully connected.} 
To compute it, we only need $S(\vectsymb{z})$, which can be obtained in $O(K)$ time with the same algorithm 
that evaluates the sparsemax. 

Often, \emph{e.g.}, in the gradient backpropagation algorithm, 
it is not necessary to compute the full Jacobian matrix, but only 
the product between the Jacobian and a given vector $\vectsymb{v}$. 
In the softmax case, from Eq.~\ref{eq:jacobian_softmax}, we have:
\begin{equation}
\matr{J}_{\mathrm{softmax}}(\vectsymb{z})\cdot\vectsymb{v} = \vectsymb{p} \odot (\vectsymb{v} - \bar{v}\vect{1}), \,\, \textstyle \text{with $\bar{v}:=\sum_{j} p_j v_j$},
\end{equation}
where $\odot$ denotes the Hadamard product; 
this requires a linear-time computation. 
For the sparsemax case, we have:
\begin{equation}
\matr{J}_{\mathrm{sparsemax}}(\vectsymb{z})\cdot\vectsymb{v} = \vectsymb{s} \odot (\vectsymb{v} - \hat{v}\vect{1}), \,\,  \text{with $\hat{v}:=\frac{\sum_{j\in S(\vectsymb{z})} v_j}{|S(\vectsymb{z})|}$.}
\end{equation}
Interestingly, if $\mathrm{sparsemax}(\vectsymb{z})$ has already been evaluated (\emph{i.e.}, in the forward step), 
then so has $S(\vectsymb{z})$, hence the nonzeros of $\matr{J}_{\mathrm{sparsemax}}(\vectsymb{z})\cdot\vectsymb{v}$ can be computed in only 
$O(|S(\vectsymb{z})|)$ time, which can be sublinear. 
This can be an important advantage of sparsemax over softmax if $K$ is large. 




\section{A Loss Function for Sparsemax}\label{sec:sparsemax_loss}

Now that we have defined the sparsemax transformation and established its main properties, 
we show how 
to use this transformation to design a new loss function 
that resembles the logistic loss, but can yield sparse posterior distributions. 
Later (in \S\ref{sec:experiments_simulations}--\ref{sec:experiments_multilabel}), we apply this loss to label proportion estimation and multi-label classification. 


\subsection{Logistic Loss}

Consider a dataset $\sett{D} := \{(\vectsymb{x}_i, y_i)\}_{i=1}^N$, 
where each $\vectsymb{x}_i \in \set{R}^D$ is an input vector and 
each $y_i \in \{1,\ldots,K\}$ is a target output label. 
We consider regularized empirical risk minimization problems 
of the form 
\begin{align}\label{eq:regularized_erm}
\mathrm{minimize}\,\,\,\, &
\frac{\lambda}{2}\|\matr{W}\|^2_{F} + \frac{1}{N} \sum_{i=1}^{N} L(\matr{W}\vectsymb{x}_i+\vectsymb{b}; y_i),\nonumber\\
\mathrm{w.r.t.} \,\,\,\, & {\matr{W} \in \set{R}^{K\times D},\,\, \vectsymb{b} \in \set{R}^K},
\end{align}
where $L$ is a loss function, $\matr{W}$ is a matrix of weights, and $\vectsymb{b}$ is a bias vector. 
The loss function associated with the softmax is the logistic loss (or negative log-likelihood):
\begin{eqnarray}\label{eq:softmax_loss}
L_{\mathrm{softmax}}(\vectsymb{z}; k) &=& -\log \mathrm{softmax}_k({\vectsymb{z}})\nonumber\\
&=& -z_k +\log \sum_j \exp (z_j),
\end{eqnarray}
where 
$\vectsymb{z} = \matr{W}\vectsymb{x}_i+\vectsymb{b}$, and 
$k=y_i$ is the ``gold'' label. 
The gradient of this loss is, invoking Eq.~\ref{eq:softmax_gradient},
\begin{equation}\label{eq:softmax_loss_gradient}
\nabla_{\vectsymb{z}} L_{\mathrm{softmax}}(\vectsymb{z}; k) 
= -\vectsymb{\delta}_k + {\mathrm{softmax}(\vectsymb{z})},
\end{equation}
where $\vectsymb{\delta}_k$ denotes the delta distribution on 
$k$, 
$[{\delta}_k]_j = 1$ if $j=k$, and $0$ otherwise. 
This is a well-known result; when plugged into a gradient-based optimizer, it leads to updates that move probability mass from the distribution predicted by the current model (\emph{i.e.}, ${\mathrm{softmax}}_k(\vectsymb{z})$) to the gold label (via $\vectsymb{\delta}_k$). Can we have something similar for sparsemax?

\subsection{Sparsemax Loss}

A nice aspect of the log-likelihood (Eq.~\ref{eq:softmax_loss}) is that adding up loss terms for several examples, assumed i.i.d, we obtain the log-probability of the full training data. Unfortunately, this idea cannot be carried out to sparsemax: now, some labels may have \emph{exactly} probability zero, so any model that assigns zero probability to a gold label would zero out the probability of the entire training sample.  
This is of course highly undesirable. One possible workaround is to define
\begin{equation}\label{eq:sparsemax_loss1}
L^{\epsilon}_{\mathrm{sparsemax}}(\vectsymb{z}; k) = -\log \frac{\epsilon + \mathrm{sparsemax}_k(\vectsymb{z})}{1+K\epsilon} ,
\end{equation}
where 
$\epsilon$ is a small constant, and 
$\frac{\epsilon + \mathrm{sparsemax}_k(\vectsymb{z})}{1+K\epsilon}$ is a ``perturbed'' 
sparsemax. However, 
this loss is non-convex, unlike 
the one in Eq.~\ref{eq:softmax_loss}.

Another possibility, which we explore here, is to 
construct 
an alternative loss function whose \emph{gradient} resembles the one in Eq.~\ref{eq:softmax_loss_gradient}. Note that the gradient is particularly important, since it is directly involved in the model updates for typical optimization algorithms. 
Formally, we want $L_{\mathrm{sparsemax}}$ to be a differentiable function such that 
\begin{eqnarray}\label{eq:sparsemax_loss_gradient}
\nabla_{\vectsymb{z}} L_{\mathrm{sparsemax}}(\vectsymb{z}; k) = 
-\vectsymb{\delta}_k + {\mathrm{sparsemax}(\vectsymb{z})}.
\end{eqnarray}

We show below that this property is fulfilled by the following function, 
henceforth called the \textbf{sparsemax loss}:
\begin{equation}\label{eq:sparsemax_loss}
L_{\mathrm{sparsemax}}(\vectsymb{z}; k) = -z_k + \frac{1}{2} \sum_{j \in S(\vectsymb{z})} (z_j^2 - \tau^2(\vectsymb{z})) + \frac{1}{2},
\end{equation}
where $\tau^2$ is the square of the threshold function in Eq.~\ref{eq:threshold_closedform}. 
This loss, which has never been considered in the literature to the best of our knowledge, has a number of interesting properties, stated in the 
next proposition. 
\begin{proposition}\label{prop:sparsemax_loss_properties}
The following holds:
\begin{enumerate}
\item $L_{\mathrm{sparsemax}}$ is differentiable everywhere, and its gradient is given by the expression in Eq.~\ref{eq:sparsemax_loss_gradient}. 
\item $L_{\mathrm{sparsemax}}$ is convex.
\item $L_{\mathrm{sparsemax}}(\vectsymb{z}+c\vect{1}; k)=L_{\mathrm{sparsemax}}(\vectsymb{z}; k)$, $\forall{c} \in \set{R}$.
\item $L_{\mathrm{sparsemax}}(\vectsymb{z}; k) \ge 0$, for all $\vectsymb{z}$ and $k$.
\item The following statements are all 
equivalent: 
(i) $L_{\mathrm{sparsemax}}(\vectsymb{z}; k) = 0$; 
(ii) ${\mathrm{sparsemax}}(\vectsymb{z}) = \vectsymb{\delta}_k$; 
(iii) margin separation holds, 
$z_k \ge 1+\max_{j \ne k} z_j$. 
\end{enumerate}
\end{proposition}
\begin{proof}
See App.~\ref{sec:proof_prop_sparsemax_loss_properties} in the supplemental material. 
\end{proof}
Note that the first four properties in Prop.~\ref{prop:sparsemax_loss_properties} are also satisfied by the logistic loss, except that the gradient is given 
by Eq.~\ref{eq:softmax_loss_gradient}. 
The fifth property is particularly interesting, since it is satisfied by the hinge loss of support vector machines. However, unlike the hinge loss, $L_{\mathrm{sparsemax}}$ is everywhere differentiable, hence amenable to smooth optimization 
methods such as L-BFGS or accelerated gradient descent \citep{Liu1989,Nesterov1983}. 


\subsection{Relation to the Huber Loss}

Coincidentally, as we next show, the sparsemax loss 
in the binary case 
reduces to the Huber classification loss, an important loss function 
in robust statistics \citep{Huber1964}. 

Let us note first that, from Eq.~\ref{eq:sparsemax_loss}, we have, if $|S(\vectsymb{z})| = 1$,
\begin{eqnarray}\label{eq:sparsemax_loss_specialcases1}
L_{\mathrm{sparsemax}}(\vectsymb{z}; k) &=& -z_k + z_{(1)},
\end{eqnarray}
and, if $|S(\vectsymb{z})| = 2$, $L_{\mathrm{sparsemax}}(\vectsymb{z}; k) = $
\begin{eqnarray}\label{eq:sparsemax_loss_specialcases2}
-z_k + \frac{1 + (z_{(1)} - z_{(2)})^2}{4} +  \frac{z_{(1)} + z_{(2)}}{2},
\end{eqnarray}
where $z_{(1)}\ge z_{(2)}\ge \ldots$ are the sorted components of $\vectsymb{z}$. 
Note that the second expression, when $z_{(1)}-z_{(2)}=1$, equals the first one, 
which asserts the continuity of the loss even though $|S(\vectsymb{z})|$ is non-continuous on $\vectsymb{z}$.

In the two-class case, 
we have $|S(\vectsymb{z})| =1$ if 
$z_{(1)} \ge 1+z_{(2)}$ (unit margin separation), 
and  $|S(\vectsymb{z})| =2$ otherwise. 
Assume without loss of generality that the correct label is $k=1$, and define $t=z_{1}-z_{2}$. 
From Eqs.~\ref{eq:sparsemax_loss_specialcases1}--\ref{eq:sparsemax_loss_specialcases2}, 
we have 
\begin{equation}
L_{\mathrm{sparsemax}}(t) = \left\{
\begin{array}{ll}
0 & \text{if $t \ge 1$}\\
-t & \text{if $t \le -1$}\\
\frac{(t - 1)^2}{4} & \text{if $-1 < t < 1$,}
\end{array}
\right.
\end{equation}
whose graph is shown in Fig.~\ref{fig:sparsemax_loss_binary_plot}. 
This loss is a variant of the Huber loss adapted for classification, and has been called ``modified Huber loss'' by \citet{Zhang2004,Zou2006}.

\begin{figure}[t]
\centering
\includegraphics[width=0.8\columnwidth]{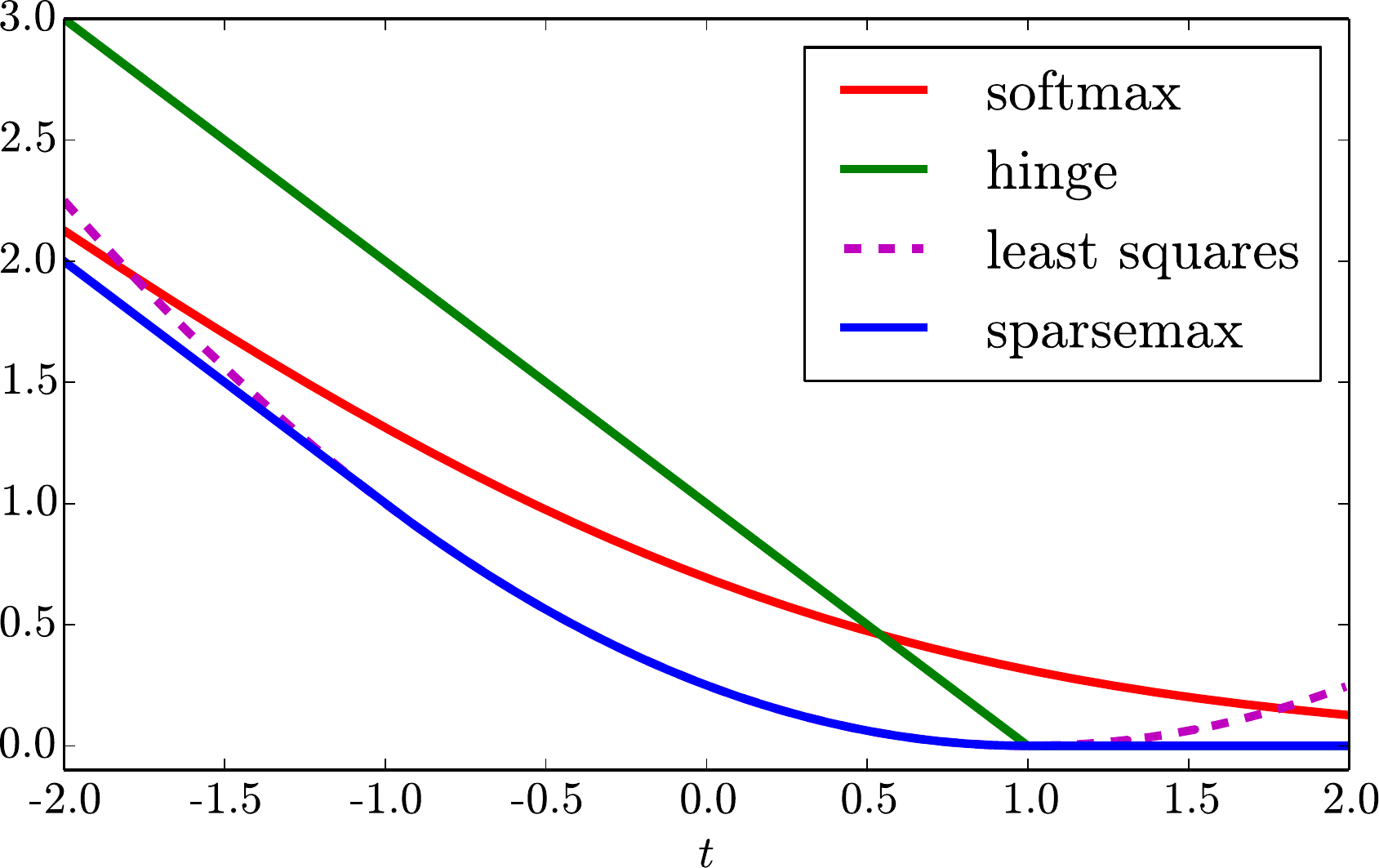}
\caption{Comparison between the sparsemax loss and other commonly used losses for binary classification.\label{fig:sparsemax_loss_binary_plot}}
\end{figure}

\subsection{Generalization to Multi-Label Classification}\label{sec:multilabel}


We end this section by showing a generalization of the loss functions in Eqs.~\ref{eq:softmax_loss} and \ref{eq:sparsemax_loss} to {{multi-label classification}}, \emph{i.e.}, problems in which the target is a non-empty set of labels $\sett{Y} \in 2^{[K]} \setminus \{\varnothing\}$ rather than a single label.%
\footnote{Not to be confused with ``multi-class classification,'' which denotes problems where 
$\sett{Y} = [K]$ and $K>2$.} %
Such problems have attracted recent interest \citep{Zhang2014multilabel}. 

More generally, we consider the problem of estimating {\textbf{sparse label proportions}}, where 
the target is a probability distribution $\vectsymb{q} \in \Delta^{K-1}$, such that $\sett{Y} = \{k \,\,|\,\, q_k > 0\}$. We assume a training dataset $\sett{D} := \{(\vectsymb{x}_i, \vectsymb{q}_i)\}_{i=1}^N$, 
where each $\vectsymb{x}_i \in \set{R}^D$ is an input vector and 
each $\vectsymb{q}_i \in \Delta^{K-1}$ is a target distribution over outputs, assumed sparse.%
\footnote{This scenario is also relevant for ``learning with a probabilistic teacher'' \citep{Agrawala1970} and semi-supervised learning \citep{Chapelle2006}, as it can model label uncertainty.} %
This subsumes single-label classification, where 
all $\vectsymb{q}_i$ are delta distributions concentrated on a single class.  
The generalization of the multinomial logistic loss to this setting is
\begin{eqnarray}\label{eq:softmax_loss_multilabel}
L_{\mathrm{softmax}}(\vectsymb{z}; \vectsymb{q}) &=& \mathbf{KL}(\vectsymb{q}\,\,\|\,\,\mathrm{softmax}({\vectsymb{z}}))\\
&=& \textstyle -\mathbf{H}(\vectsymb{q})-\DP{\vectsymb{q}}{\vectsymb{z}} +\log \sum_j \exp (z_j),\nonumber
\end{eqnarray}
where $\mathbf{KL}(.\|.)$ and $\mathbf{H}(.)$ denote the Kullback-Leibler divergence and the Shannon entropy, respectively. 
Note that, up to a constant, this loss is equivalent to standard logistic regression with soft labels. 
The gradient of this loss is
\begin{equation}\label{eq:softmax_loss_gradient_multilabel}
\nabla_{\vectsymb{z}} L_{\mathrm{softmax}}(\vectsymb{z}; \vectsymb{q}) = -\vectsymb{q} + {\mathrm{softmax}(\vectsymb{z})}.
\end{equation} 
The corresponding generalization in the sparsemax case is:
\begin{equation}\label{eq:sparsemax_loss_multilabel}
L_{\mathrm{sparsemax}}(\vectsymb{z}; \vectsymb{q}) = -\DP{\vectsymb{q}}{\vectsymb{z}} + \frac{1}{2} \sum_{j \in S(\vectsymb{z})} (z_j^2 - \tau^2(\vectsymb{z})) + \frac{1}{2}\|\vectsymb{q}\|^2,
\end{equation}
which satisfies the properties in Prop.~\ref{prop:sparsemax_loss_properties} and has gradient 
\begin{eqnarray}\label{eq:sparsemax_loss_gradient_multilabel}
\nabla_{\vectsymb{z}} L_{\mathrm{sparsemax}}(\vectsymb{z}; \vectsymb{q})= 
-\vectsymb{q} + {\mathrm{sparsemax}(\vectsymb{z})}. 
\end{eqnarray}
We make use of these losses in our experiments (\S\ref{sec:experiments_simulations}--\ref{sec:experiments_multilabel}). 

\section{Experiments}\label{sec:experiments}
\label{experiments}

We next evaluate empirically the ability of sparsemax for addressing two classes of problems:
\begin{enumerate}
\item Label proportion estimation and multi-label classification, 
via the sparsemax loss in Eq.~\ref{eq:sparsemax_loss_multilabel}  (\S\ref{sec:experiments_simulations}--\ref{sec:experiments_multilabel}).
\item Attention-based neural networks, via the sparsemax transformation of Eq.~\ref{eq:sparsemax}  (\S\ref{sec:experiments_attention}).  
\end{enumerate}


\begin{SCfigure*}[\sidecaptionrelwidth][t]
\centering
\includegraphics[width=0.7\textwidth]{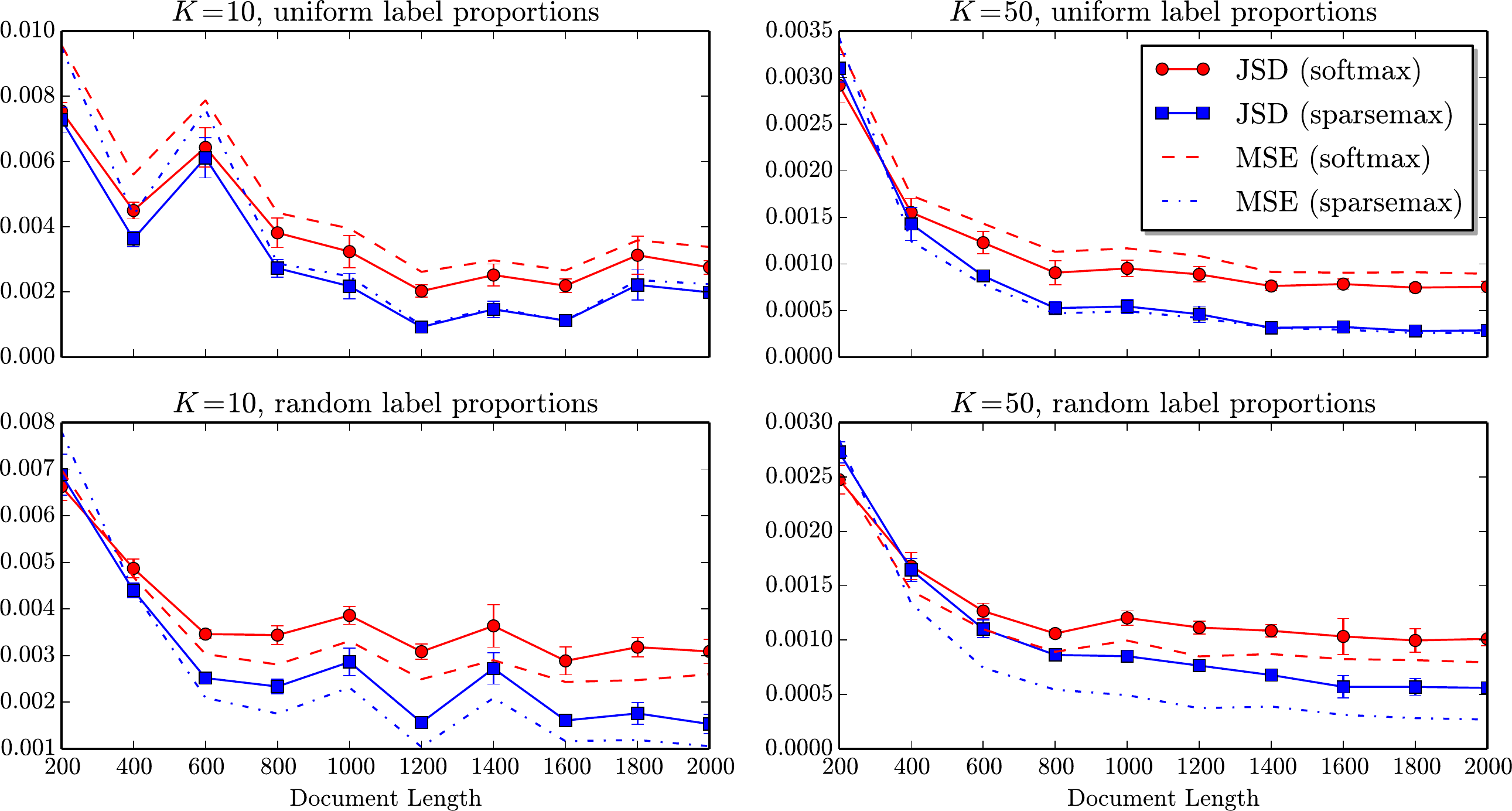}
\caption{Simulation results for the estimation of label posteriors, for uniform (top) and random mixtures (bottom). Shown are the mean squared error and the Jensen-Shannon divergence 
as a function of the document length, for the logistic and the sparsemax estimators.}\label{fig:toy_experiments}
\end{SCfigure*}

\subsection{Label Proportion Estimation}\label{sec:experiments_simulations}



We show simulation results for 
sparse label proportion estimation 
on synthetic data. 
Since sparsemax can predict sparse distributions, we expect 
its superiority in this task. 
We generated datasets with 1,200 training and 1,000 test examples. 
Each example emulates a ``multi-labeled document'': 
a variable-length sequence of word symbols, assigned to 
multiple topics (labels). 
We pick the number of labels $N\in \{1,\ldots,K\}$ by sampling from a Poisson distribution with rejection sampling, and draw the $N$ labels from a multinomial. Then, we pick a document length from a Poisson, and repeatedly sample its words from the mixture of the $N$ label-specific multinomials. We experimented with two settings: {\textbf{uniform mixtures}} ($q_{k_n} = 1/N$ for the $N$ active labels $k_1,\ldots,k_N$) and 
{\textbf{random mixtures}} (whose label proportions $q_{k_n}$ were drawn from a flat Dirichlet).%
\footnote{Note that, with uniform mixtures, the problem becomes essentially multi-label classification.} %
We set the vocabulary size to be equal to the number of labels $K \in \{10, 50\}$, and varied the average document length between 200 and 2,000 words. 
We trained models by optimizing Eq.~\ref{eq:regularized_erm} with $L \in \{L_{\mathrm{softmax}}, L_{\mathrm{sparsemax}}\}$ (Eqs.~\ref{eq:softmax_loss_multilabel} and \ref{eq:sparsemax_loss_multilabel}). We picked the regularization constant $\lambda \in \{10^j\}_{j=-9}^0$ with 5-fold cross-validation.

Results are shown in Fig.~\ref{fig:toy_experiments}. 
We report the {{mean squared error}} (average of $\|\vectsymb{q}-\vectsymb{p}\|^2$ on the test set, where $\vectsymb{q}$ and $\vectsymb{p}$ are respectively the target and predicted label posteriors) and the {{Jensen-Shannon divergence}} 
(average of $\mathbf{JS}(\vectsymb{q},\vectsymb{p}) := \frac{1}{2}\mathbf{KL}(\vectsymb{q}\|\frac{\vectsymb{p}+\vectsymb{q}}{2}) + \frac{1}{2}\mathbf{KL}(\vectsymb{p}\|\frac{\vectsymb{p}+\vectsymb{q}}{2})$).%
\footnote{Note that the KL divergence is not an appropriate metric here, since the sparsity of $\vectsymb{q}$ and $\vectsymb{p}$ could lead to $-\infty$ values.} %
We observe that the two losses perform similarly for small document lengths (where the signal is weaker), but as the average document length exceeds 400, 
the sparsemax loss starts outperforming the logistic loss consistently. 
This is because with a stronger signal the sparsemax estimator manages to identify correctly the support of the label proportions $\vectsymb{q}$, contributing to reduce both the mean squared error and the JS divergence. This occurs both for uniform and random mixtures. 

\subsection{Multi-Label Classification on Benchmark Datasets}\label{sec:experiments_multilabel}

Next, 
we ran experiments in five benchmark multi-label classification datasets: 
the four small-scale datasets used by \citet{Koyejo2015},%
\footnote{Obtained from \url{http://mulan.sourceforge.net/datasets-mlc.html}.} %
and the much larger Reuters RCV1 v2 dataset of \citet{Lewis2004}.%
\footnote{Obtained from \url{https://www.csie.ntu.edu.tw/~cjlin/libsvmtools/datasets/multilabel.html}.} %
For all datasets, we removed examples without labels (\emph{i.e.} where $\sett{Y}=\varnothing$).  
For all but the Reuters dataset, we normalized the features to have zero mean and unit variance. 
Statistics for these datasets are presented in Table~\ref{tab:datasets}. 

\begin{table}[t]
\caption{Statistics for the 5 multi-label classification datasets.\label{tab:datasets}}
\label{sample-table}
\vspace{0.1cm}
\begin{center}
\begin{small}
\begin{sc}
\begin{tabular}{lccrr}
\belowspace
Dataset & Descr. & \#labels & \#train & \#test \\
\hline
\abovespace
Scene & images & 6 & 1211 & 1196 \\
Emotions & music & 6 & 393 & 202 \\
Birds & audio & 19 & 323 & 322 \\
CAL500 & music & 174 & 400 & 100 \\ 
\belowspace
Reuters & text & 103 & 23,149 & 781,265 \\
\hline
\end{tabular}
\end{sc}
\end{small}
\end{center}
\vskip -0.1in
\end{table}

\begin{table}[t]
\caption{Micro (left) and macro-averaged (right) F$_1$ scores for the logistic, softmax, and sparsemax losses on benchmark 
datasets.}
\label{table:exp_multilabel}
\vspace{-0.2cm}
\begin{center}
\begin{small}
\begin{sc}
\begin{tabular}{l@{\hskip 8pt}c@{\hskip 8pt}c@{\hskip 8pt}c@{\hskip 8pt}r}
\belowspace
Dataset & Logistic & Softmax & Sparsemax \\
\hline
\abovespace
Scene & 70.96 / 72.95 & {\bf 74.01} / {\bf 75.03} & 73.45 / 74.57\\
Emotions & 66.75 / {\bf 68.56} & {\bf 67.34} / 67.51 & 66.38 / 66.07\\
Birds & 45.78 / 33.77 & 48.67 / 37.06 & {\bf 49.44} / {\bf 39.13}\\
CAL500 & {\bf 48.88} / 24.49 & 47.46 / 23.51 & 48.47 / {\bf 26.20}\\ 
\belowspace
Reuters & {\bf 81.19} / 60.02 & 79.47 / 56.30 & 80.00 / {\bf 61.27}\\
\hline
\end{tabular}
\end{sc}
\end{small}
\end{center}
\vspace{-0.5cm}
\end{table}

Recent work has investigated the consistency of multi-label classifiers  for various micro and macro-averaged metrics \citep{Gao2013,Koyejo2015}, among which a plug-in classifier that trains independent binary logistic regressors on each label, and then tunes a probability threshold $\delta \in [0,1]$ on validation data. At test time, those labels whose posteriors are above the threshold are predicted to be ``on.'' 
We used this procedure (called {\sc{Logistic}}) as a baseline for comparison.
Our second baseline ({\sc{Softmax}}) is a multinomial logistic regressor, using the loss function in 
Eq.~\ref{eq:softmax_loss_multilabel}, where the target distribution $\vectsymb{q}$ is set to uniform over the active labels. A similar probability threshold $p_0$ is used for prediction, above which a label is predicted to be ``on.'' 
We compare these two systems with the sparsemax loss function of 
Eq.~\ref{eq:sparsemax_loss_multilabel}. We found it beneficial to scale 
the label scores $\vectsymb{z}$ by a constant $t\ge 1$ at test time, before applying the sparsemax transformation, to make the resulting distribution $\vectsymb{p} = \mathrm{sparsemax}(t\vectsymb{z})$ more sparse. We then predict the $k$th label to be ``on'' if $p_k \ne 0$.  

We optimized the three losses with L-BFGS 
(for a maximum of 100 epochs), tuning the hyperparameters in a held-out validation set (for the Reuters dataset) and with 5-fold cross-validation (for the other four datasets). 
The hyperparameters are the regularization constant 
$\lambda \in \{10^{j}\}_{j=-8}^2$, 
the probability thresholds 
$\delta \in \{.05\times n\}_{n=1}^{10}$ for {\sc{Logistic}} 
and $p_0 \in \{n/K\}_{n=1}^{10}$ for {\sc{Softmax}}, 
and the coefficient 
$t \in \{0.5\times n\}_{n=1}^{10}$ for {\sc{Sparsemax}}.

Results are shown in Table~\ref{table:exp_multilabel}. 
Overall, the performances of the three losses are all very similar, with a slight advantage of {\sc{Sparsemax}}, which  attained the highest results in 4 out of 10 experiments, while {\sc{Logistic}} and {\sc{Softmax}} won 
3 times each. 
In particular, sparsemax appears better suited for problems with larger numbers of labels.

\subsection{Neural Networks with Attention Mechanisms}\label{sec:experiments_attention}

We now assess the suitability of the sparsemax transformation to construct a ``sparse'' neural attention mechanism. 

We ran experiments on the task of natural language inference, using the recently released SNLI 1.0 corpus \citep{Bowman2015}, a 
collection of 570,000 human-written English sentence pairs. Each pair consists of a premise and an hypothesis, manually labeled with 
one the labels {\sc entailment}, {\sc contradiction}, or {\sc neutral}. 
We used the provided training, development, and test splits. %

The architecture of our system, shown in Fig.~\ref{fig:neural_snli}, is the same as the one proposed by \citet{Rocktaschel2015}. 
We compare the performance of four systems: {\sc{NoAttention}}, a (gated) RNN-based system similar to \citet{Bowman2015}; 
{\sc{LogisticAttention}}, 
an attention-based system with independent logistic activations; 
{\sc{SoftAttention}}, 
a near-reproduction of the \citet{Rocktaschel2015}'s attention-based system; 
and {\sc{SparseAttention}}, which replaces the latter softmax-activated attention mechanism  
by a sparsemax activation. 
\begin{figure}[t]
\centering
\includegraphics[width=\columnwidth]{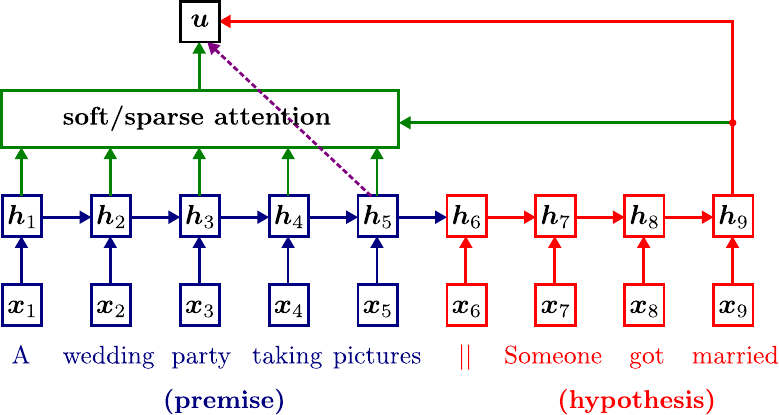}
\caption{Network diagram for the NL inference problem. The premise and hypothesis are both fed into (gated) RNNs. The {\sc{NoAttention}} system replaces the attention part (in green) by a direct connection from the last premise state to the output (dashed violet line). The {\sc{LogisticAttention}}, {\sc{SoftAttention}} and {\sc{SparseAttention}} systems have respectively independent logistics, a softmax, and a sparsemax-activated attention mechanism. In this example, $L=5$ and $N=9$.\label{fig:neural_snli}}
\end{figure}

We represent the words in the premise and in the hypothesis with 
300-dimensional GloVe vectors \citep{Pennington2014}, 
not optimized during training, 
which we linearly project onto a $D$-dimensional subspace \citep{Astudillo2015}.%
\footnote{We used GloVe-840B embeddings trained on Common Crawl 
(\url{http://nlp.stanford.edu/projects/glove/}).} %
We denote by $\vectsymb{x}_1,\ldots,\vectsymb{x}_L$ and 
$\vectsymb{x}_{L+1},\ldots,\vectsymb{x}_N$, respectively, the projected premise and hypothesis 
word vectors. 
These sequences are then fed into two recurrent networks (one for each). 
Instead of long short-term memories, as \citet{Rocktaschel2015},
we used 
gated recurrent units (GRUs, \citealt{Cho2014}), which behave similarly but have fewer parameters. 
Our premise GRU generates a state sequence 
$\matr{H}_{1:L} := [\vectsymb{h}_1 \ldots \vectsymb{h}_L] \in \set{R}^{D\times L}$ as follows: 
\begin{eqnarray}
\vectsymb{z}_t &\!\!\!\!=\!\!\!\!& \sigma(\matr{W}^{xz} \vectsymb{x}_t + \matr{W}^{hz} \vectsymb{h}_{t-1} + \vectsymb{b}^z)\\
\vectsymb{r}_t &\!\!\!\!=\!\!\!\!& \sigma(\matr{W}^{xr} \vectsymb{x}_t + \matr{W}^{hr} \vectsymb{h}_{t-1} + \vectsymb{b}^r)\\
{\bar{\vectsymb{h}}}_t &\!\!\!\!=\!\!\!\!& \tanh(\matr{W}^{xh} \vectsymb{x}_t + \matr{W}^{hh} (\vectsymb{r}_t \odot \vectsymb{h}_{t-1}) + \vectsymb{b}^h)\\
\vectsymb{h}_t &\!\!\!\!=\!\!\!\!& (1-\vectsymb{z}_t) \vectsymb{h}_{t-1} +  \vectsymb{z}_t{\bar{\vectsymb{h}}}_t,
\end{eqnarray}
with model parameters 
$\matr{W}^{\{{xz}, {xr}, {xh}, {hz}, {hr}, {hh}\}} \in \set{R}^{D\times D}$ 
and  
$\vectsymb{b}^{\{{z}, r, h\}} \in \set{R}^D$.  
Likewise, our hypothesis GRU (with distinct parameters) generates a state sequence  $[\vectsymb{h}_{L+1},\ldots,\vectsymb{h}_N]$, being initialized with the last state from the premise ($\vectsymb{h}_{L}$). 
The {\sc{NoAttention}} system 
then computes the final state $\vectsymb{u}$
based on the last states from the 
premise and the hypothesis as follows:
\begin{eqnarray}
\vectsymb{u} = \tanh(\matr{W}^{pu}\vectsymb{h}_L + \matr{W}^{hu}\vectsymb{h}_N + \vectsymb{b}^u)\label{eq:snli_output_state}
\end{eqnarray}
where 
$\matr{W}^{pu},\matr{W}^{hu} \in \set{R}^{D\times D}$ and $\vectsymb{b}^u \in \set{R}^D$.   
Finally, it predicts a label ${\widehat{y}}$ from $\vectsymb{u}$ with a standard softmax layer. 
The {\sc{SoftAttention}} system, instead of using the last premise state $\vectsymb{h}_L$, computes a weighted average of premise words with an attention mechanism, replacing Eq.~\ref{eq:snli_output_state} by
\begin{eqnarray}
z_t &=& \DP{\vectsymb{v}}{\tanh(\matr{W}^{pm}\vectsymb{h}_{t} + \matr{W}^{hm}\vectsymb{h}_N  + \vectsymb{b}^m)}\\
\vectsymb{p} &=& \mathrm{softmax}(\vectsymb{z}),\,\, \text{where $\vectsymb{z}:=(z_1,\ldots,z_L)$} \label{eq:snli_soft_attention}\\
\vectsymb{r} &=& \matr{H}_{1:L}\vectsymb{p}\\
\vectsymb{u} &=& \tanh(\matr{W}^{pu}\vectsymb{r} + \matr{W}^{hu}\vectsymb{h}_N + \vectsymb{b}^u),
\end{eqnarray}
where $\matr{W}^{pm}, \matr{W}^{hm} \in \set{R}^{D\times D}$ and $\vectsymb{b}^m,\vectsymb{v} \in \set{R}^D$. 
The {\sc{LogisticAttention}} system, instead of Eq.~\ref{eq:snli_soft_attention}, 
computes $\vectsymb{p} = (\sigma(z_1),\ldots,\sigma(z_L))$.  
Finally, the  {\sc{SparseAttention}} system replaces Eq.~\ref{eq:snli_soft_attention} by
$\vectsymb{p} = \mathrm{sparsemax}(\vectsymb{z})$.

We optimized all the systems with Adam \citep{Kingma2014}, using the default parameters $\beta_1=0.9$, $\beta_2=0.999$, and $\epsilon=10^{-8}$, and setting the learning rate to $3\times 10^{-4}$. 
We tuned a $\ell_2$-regularization coefficient in $\{0,10^{-4},3\times 10^{-4},10^{-3}\}$ and, as 
\citet{Rocktaschel2015}, a dropout probability of $0.1$
in the inputs and outputs of the network. 

\begin{table}[t]
\caption{Accuracies for the natural language inference task. Shown 
are our implementations of a system without attention, and with logistic, soft, and sparse attentions.} 

\label{table:exp_snli}
\begin{center}
\begin{small}
\begin{sc}
\begin{tabular}{l@{\hskip 8pt}c@{\hskip 8pt}c}
\belowspace
 & Dev Acc. & Test Acc. \\
\hline
\abovespace
{\sc{NoAttention}} & 81.84 & 80.99 \\
{\sc{LogisticAttention}} & 82.11 & 80.84 \\
{\sc{SoftAttention}} & 82.86 & 82.08 \\
\belowspace
{\sc{SparseAttention}} & 82.52 & {\textbf{82.20}} \\
\hline
\end{tabular}
\end{sc}
\end{small}
\end{center}
\vskip -0.1in
\end{table}

\begin{table}[t]
\caption{Examples of sparse attention for the natural language inference task. Nonzero attention coefficients are marked {\textbf{in bold}}. Our system classified all four examples correctly. The examples were picked from \citet{Rocktaschel2015}.} 
\label{table:exp_sparse_snli}
\begin{center}
\begin{small}
\begin{tabular}{@{\hskip 0pt}p{\columnwidth}@{\hskip 0pt}}
\hline
\abovespace
{A boy {\textbf{rides on}} a {\textbf{camel}} in a crowded area while talking on his cellphone.} \\
\belowspace
Hypothesis: {\it{A boy is riding an animal.}} \hfill [entailment]\\
\hline
\abovespace
A young girl wearing {\textbf{a pink coat}} plays with a {\textbf{yellow}} toy golf club. \\
\belowspace
Hypothesis: {\it{A girl is wearing a blue jacket.}} \hfill [contradiction]\\
\hline
\abovespace
Two black dogs are {\textbf{frolicking}} around the {\textbf{grass together}}.\\
\belowspace
Hypothesis: {\it{Two dogs swim in the lake.}} \hfill [contradiction]\\
\hline
\abovespace
A man wearing a yellow striped shirt {\textbf{laughs}} while {\textbf{seated next}} to another {\textbf{man}} who is wearing a light blue shirt and {\textbf{clasping}} his hands together.\\
\belowspace
Hypothesis: {\it{Two mimes sit in complete silence.}} \hfill [contradiction]\\
\hline
\end{tabular}
\end{small}
\end{center}
\vskip -0.1in
\end{table}

The results are shown in Table~\ref{table:exp_snli}. We observe that the soft and sparse-activated attention systems perform similarly, the latter being slightly more accurate on the test set, 
and that both outperform the {\sc{NoAttention}} and {\sc{LogisticAttention}} systems.%
\footnote{\citet{Rocktaschel2015} report scores slightly above ours: they reached a test accuracy of 82.3\% for their implementation of {\sc{SoftAttention}}, and 83.5\% with their best system, a more elaborate word-by-word attention model. Differences in the former case may be due to distinct word vectors and the use of LSTMs instead of GRUs.} %

Table~\ref{table:exp_sparse_snli} shows examples of sentence pairs, highlighting the premise words selected by the {\sc{SparseAttention}} mechanism. We can see that, for all examples, only a small number of words are selected, which are key to making the final decision. 
Compared to a softmax-activated mechanism, which provides a dense distribution over all the words, 
the sparsemax activation yields a compact and more interpretable selection, which can be particularly useful in long sentences such as the one in the bottom row.

\section{Conclusions}\label{sec:conclusions}

We introduced the sparsemax transformation, which has similar properties to the traditional softmax, but is able to output sparse probability distributions. 
We derived a closed-form expression for its Jacobian, needed for the backpropagation algorithm, and we proposed a novel ``sparsemax loss'' function, a sparse analogue of the logistic loss, which is smooth and convex.  
Empirical results in multi-label classification and in attention networks for natural language inference attest the validity of our approach. 

The connection between sparse modeling and interpretability is key in 
signal processing \citep{Hastie2015}. 
Our approach is distinctive: it is not the model that is assumed sparse, 
but the label posteriors that the model parametrizes. 
Sparsity is also a desirable (and biologically plausible) property in neural networks, 
present in rectified units 
\citep{Glorot2011AISTATS} and maxout nets \citep{Goodfellow2013}.  


There are several avenues for future research. 
The ability of sparsemax-activated attention to select only a few variables to attend makes it potentially relevant to 
neural architectures with random access memory 
\citep{Graves2014,Grefenstette2015,Sukhbaatar2015},  
since it offers a compromise between soft and hard operations, maintaining differentiability. 
In fact, ``harder'' forms of attention are often useful, arising 
as word alignments in 
machine translation pipelines, 
or latent variables as in \citet{Xu2015}.  
Sparsemax is also appealing for hierarchical attention: if we define a top-down product of distributions along the hierarchy, the sparse distributions produced by sparsemax will automatically prune the hierarchy, leading to computational savings. 
A possible disadvantage of sparsemax over softmax is that it seems less GPU-friendly, since it requires sort operations or linear-selection algorithms. 
There is, however, recent work providing efficient implementations of these algorithms on GPUs \citep{Alabi2012}. 


\section*{Acknowledgements} 
We would like to thank Tim Rockt\"aschel for answering various implementation questions, and M\'ario Figueiredo and Chris Dyer for helpful comments on a draft of this report. 
This work was partially supported by Funda\c{c}\~ao para a Ci\^{e}ncia e Tecnologia (FCT), through contracts  UID/EEA/50008/2013 and UID/CEC/50021/2013.

\bibliography{icml2016}
\bibliographystyle{icml2016}

\clearpage

\appendix
\onecolumn

\section{Supplementary Material}

\subsection{Proof of Prop.~\ref{prop:sparsemax_solution}}\label{sec:proof_threshold_closedform}

The Lagrangian of the optimization problem in Eq.~\ref{eq:sparsemax} is:
\begin{equation}
\sett{L}(\vectsymb{z}, \vectsymb{\mu}, \tau) = \frac{1}{2}\|\vectsymb{p}-\vectsymb{z}\|^2 - \DP{\vectsymb{\mu}}{\vectsymb{p}} + \tau (\DP{\vect{1}}{\vectsymb{p}} - 1).
\end{equation}
The optimal $(\vectsymb{p}^*, \vectsymb{\mu}^*, \tau^*)$ must satisfy the following 
Karush-Kuhn-Tucker conditions:
\begin{eqnarray}
\vectsymb{p}^* - \vectsymb{z} - \vectsymb{\mu}^* + \tau^*\vect{1} = \vect{0}\label{eq:proof_zero_lagrangian}, \\
\DP{\vect{1}}{\vectsymb{p}^*} = 1, \,\,\,\, \vectsymb{p}^* \ge \vect{0}, \,\,\,\, \vectsymb{\mu}^* \ge \vect{0},\label{eq:proof_primal_dual_feasibility}\\
\mu_i^* p_i^* = 0, \,\, \forall i \in [K].\label{eq:proof_complementary_slackness}
\end{eqnarray}
If for $i \in [K]$ we have $p_i^*>0$, then from Eq.~\ref{eq:proof_complementary_slackness} we must have 
$\mu_i^* = 0$, which from Eq.~\ref{eq:proof_zero_lagrangian} implies $p_i^* = z_i - \tau^*$. 
Let $S(\vectsymb{z})=\{j \in [K]\,\,|\,\, p_j^* > 0\}$. 
From Eq.~\ref{eq:proof_primal_dual_feasibility} we obtain 
$\sum_{j \in S(\vectsymb{z})} (z_j - \tau^*)=1$, 
which yields the right hand side of Eq.~\ref{eq:threshold_closedform}. 
Again from Eq.~\ref{eq:proof_complementary_slackness}, we have that $\mu_i^* > 0$ implies 
$p_i^*=0$, which from Eq.~\ref{eq:proof_zero_lagrangian} implies $\mu_i^* = \tau^* - z_i \ge 0$, i.e., 
$z_i \le \tau^*$ for $i \notin S(\vectsymb{z})$. Therefore we have that $k(\vectsymb{z}) = |S(\vectsymb{z})|$, which proves the first equality of 
Eq.~\ref{eq:threshold_closedform}.

%

\subsection{Proof of Prop.~\ref{prop:sparsemax_softmax_properties}}\label{sec:proof_prop_sparsemax_softmax_properties}


We start with the third property, which follows from the coordinate-symmetry in the definitions in Eqs.~\ref{eq:softmax}--\ref{eq:sparsemax}. The same argument can be used to prove the first part of the first property (uniform distribution). 

Let us turn to the second part of the first property (peaked distribution on the maximal components of $\vectsymb{z}$), and define $t=\epsilon^{-1}$. For the softmax case, this follows from 
\begin{equation}
\lim_{t\rightarrow +\infty} \frac{e^{tz_i}}{\sum_k e^{tz_k}} = \lim_{t\rightarrow +\infty} \frac{e^{tz_i}}{\sum_{k\in A(\vectsymb{z})} e^{tz_k}} = \lim_{t\rightarrow +\infty} \frac{e^{t(z_i-z_{(1)})}}{|A(\vectsymb{z})|}
= 
\left\{
\begin{array}{ll}
1/|A(\vectsymb{z})|, & \text{if $i \in A(\vectsymb{z})$}\\
0, & \text{otherwise.} 
\end{array}
\right.
\end{equation}
For the sparsemax case, we invoke Eq.~\ref{eq:threshold_closedform} and the fact that $k(t\vectsymb{z})=|A(\vectsymb{z})|$ if $\gamma(t\vectsymb{z})\ge 1/|A(\vectsymb{z})|$. 
Since $\gamma(t\vectsymb{z})=t\gamma(\vectsymb{z})$, the result follows.


The second property holds for softmax, since $(e^{z_i+c})/\sum_k e^{z_k+c} = e^{z_i}/\sum_k e^{z_k}$; and for sparsemax, since for any $\vectsymb{p} \in \Delta^{K-1}$ we have $\|\vectsymb{p} - \vectsymb{z} - c\vect{1}\|^2 = \|\vectsymb{p} - \vectsymb{z}\|^2 -2c\DP{\vect{1}}{(\vectsymb{p} - \vectsymb{z})} + \|c\vect{1}\|^2$, which equals $\|\vectsymb{p} - \vectsymb{z}\|^2$ plus a constant (because $\DP{\vect{1}}{\vectsymb{p}} = 1$).

Finally, let us turn to fourth property. The first inequality states that $z_i \le z_j \,\, \Rightarrow \,\, \rho_i(\vectsymb{z}) \le \rho_j(\vectsymb{z})$ (\emph{i.e.}, coordinate monotonicity). For the softmax case, this follows trivially from the fact that the exponential function is increasing. 
For the sparsemax, we use a proof by contradiction. 
Suppose $z_i \le z_j$ and $\mathrm{sparsemax}_i(\vectsymb{z}) > \mathrm{sparsemax}_j(\vectsymb{z})$. 
From the definition in Eq.~\ref{eq:sparsemax}, 
we must have $\|\vectsymb{p} - \vectsymb{z}\|^2 \ge \|\mathrm{sparsemax}(\vectsymb{z}) - \vectsymb{z}\|^2$, for any $\vectsymb{p} \in \Delta^{K-1}$. 
This leads to a contradiction if we choose $p_k = \mathrm{sparsemax}_k(\vectsymb{z})$ for $k \notin \{i,j\}$, $p_i = \mathrm{sparsemax}_j(\vectsymb{z})$, and 
$p_j = \mathrm{sparsemax}_i(\vectsymb{z})$. 
To prove the second inequality in the fourth property for softmax, we need to show that, with $z_i \le z_j$, we have $(e^{z_j}-e^{z_i})/\sum_k e^{z_k} \le (z_j - z_i)/2$. Since $\sum_k e^{z_k} \ge e^{z_j}+e^{z_i}$, it suffices to consider the binary case, \emph{i.e.},  we need to prove that $\tanh((z_j-z_i)/2) = (e^{z_j}-e^{z_i})/(e^{z_j}+e^{z_i}) \le (z_j - z_i)/2$, that is,
$\tanh(t) \le t$ for $t\ge 0$. This comes from $\tanh(0)=0$ and $\tanh'(t) = 1-\tanh^2(t) \le 1$. 
For sparsemax, given two coordinates $i$, $j$, three things can happen: 
(i) both are thresholded, in which case $\rho_j(\vectsymb{z})-\rho_i(\vectsymb{z})=z_j-z_i$; (ii) the smaller ($z_i$) is truncated, in which case $\rho_j(\vectsymb{z})-\rho_i(\vectsymb{z})=z_j-\tau(\vectsymb{z})\le z_j-z_i$; 
(iii) both are truncated, in which case $\rho_j(\vectsymb{z})-\rho_i(\vectsymb{z})=0\le z_j-z_i$.

\subsection{Proof of Prop.~\ref{prop:sparsemax_loss_properties}}\label{sec:proof_prop_sparsemax_loss_properties}

To prove the first claim, note that, for $j \in S(\vectsymb{z})$,
\begin{equation}
\frac{\partial \tau^2(\vectsymb{z})}{\partial z_j} = 2 \tau(\vectsymb{z}) \frac{\partial \tau(\vectsymb{z})}{\partial z_j}= \frac{2 \tau(\vectsymb{z})}{|S(\vectsymb{z})|},
\end{equation} 
where we used Eq.~\ref{eq:gradient_tau}. We then have
\begin{eqnarray}
\frac{\partial L_{\mathrm{sparsemax}}(\vectsymb{z}; k)}{\partial z_j} = 
\left\{\!\!
\begin{array}{ll} 
-{\delta}_{k}(j) + z_j - \tau(\vectsymb{z}) & \text{if $j \in S(\vectsymb{z})$}\\
-{\delta}_{k}(j) & \text{otherwise}.
\end{array}
\right.\nonumber
\end{eqnarray} 
That is,
$\nabla_{\vectsymb{z}} L_{\mathrm{sparsemax}}(\vectsymb{z};k) = 
-\vectsymb{\delta}_k + {\mathrm{sparsemax}(\vectsymb{z})}$. 

To prove the second statement, from the expression for the Jacobian in Eq.~\ref{eq:sparsemax_gradient}, we have that the Hessian of $L_{\mathrm{sparsemax}}$ 
(strictly speaking, a ``sub-Hessian'' \citep{Penot2014}, since the loss is not twice-differentiable everywhere) is given by
\begin{eqnarray}\label{eq:sparsemax_loss_hessian}
\frac{\partial^2 L_{\mathrm{sparsemax}}(\vectsymb{z}; k)}{\partial x_i \partial x_j} = 
\left\{
\begin{array}{ll}
\delta_{ij}-\frac{1}{|S(\vectsymb{z})|} & \text{if $i,j\in S(\vectsymb{z})$}\\
0 & \text{otherwise.}
\end{array}
\right.
\end{eqnarray}
This Hessian can be written in the form $\mathrm{Id} - \vect{1}\vect{1}^{\top}/|S(\vectsymb{z})|$ up to padding zeros (for the coordinates not in $S(\vectsymb{z})$); hence it is positive semi-definite (with rank $|S(\vectsymb{z})|-1$), 
which establishes the convexity of $L_{\mathrm{sparsemax}}$. 

For the third claim, we have $L_{\mathrm{sparsemax}}(\vectsymb{z}+c\vect{1})= -z_k - c + \frac{1}{2} \sum_{j \in S(\vectsymb{z})} (z_j^2 - \tau^2(\vectsymb{z}) +2c(z_j-\tau)) + \frac{1}{2}
= -z_k - c + \frac{1}{2} \sum_{j \in S(\vectsymb{z})} (z_j^2 - \tau^2(\vectsymb{z}) +2c p_j) + \frac{1}{2} =L_{\mathrm{sparsemax}}(\vectsymb{z})$, since $\sum_{j \in S(\vectsymb{z})} p_j = 1$. 

From the first two claims, we have that the minima of $L_{\mathrm{sparsemax}}$ have zero gradient, i.e., satisfy the equation 
${\mathrm{sparsemax}(\vectsymb{z})} = \vectsymb{\delta}_k$.  Furthemore, 
from Prop.~\ref{prop:sparsemax_softmax_properties}, we have that the sparsemax never increases the distance between two coordinates, i.e., $\mathrm{sparsemax}_k(\vectsymb{z}) - \mathrm{sparsemax}_j(\vectsymb{z}) \le z_k - z_j$. 
Therefore ${\mathrm{sparsemax}(\vectsymb{z})} = \vectsymb{\delta}_k$
implies $z_k \ge 1+\max_{j \ne k} z_j$. 
To prove the converse statement, note that the distance above can only be decreased
if the smallest coordinate is truncated to zero. 
This establishes the equivalence between (ii) and (iii) in the fifth claim. 
Finally, we have that the minimum loss value is achieved when $S(\vectsymb{z}) = \{k\}$, in which case $\tau(\vectsymb{z}) = z_k - 1$, leading to
\begin{eqnarray}
L_{\mathrm{sparsemax}}(\vectsymb{z}; k) = -z_k + \frac{1}{2}  (z_k^2 - (z_k-1)^2) + \frac{1}{2}=0.
\end{eqnarray}
This proves the equivalence with (i) and also the fourth claim. 

\end{document}